\documentclass[10pt,twocolumn,letterpaper]{article}
\pdfoutput=1
\usepackage{iccv}
\usepackage{times}
\usepackage{epsfig}
\usepackage{graphicx}
\usepackage{amsmath}
\usepackage{amssymb}
\usepackage[misc]{ifsym}
\usepackage{comment}
\usepackage{color}
\usepackage{amsthm}
\usepackage{subfigure}
\usepackage{algorithm}
\usepackage{algorithmic}
\usepackage{array}
\usepackage{booktabs} 
\usepackage{makecell,multirow,diagbox}
\usepackage{makecell}
\newtheorem*{theorem}{Theorem}
\newtheorem*{corollary}{Corollary}

\usepackage[breaklinks=true,bookmarks=false]{hyperref}

\iccvfinalcopy 

\def\iccvPaperID{****} 

\ificcvfinal\fi

\begin{document}

\title{Probabilistic Ranking-Aware Ensembles for Enhanced  Object Detections}


\author{Mingyuan Mao\textsuperscript{1}, 
 Baochang Zhang\textsuperscript{1,*}, David Doermann\textsuperscript{2}, Jie Guo\textsuperscript{1},\\ 
 Shumin Han\textsuperscript{3,*}, Yuan Feng\textsuperscript{3}, Xiaodi Wang\textsuperscript{3}, Errui Ding\textsuperscript{3} \\

\textsuperscript{1}Beihang University, Beijing, China\\
\textsuperscript{2}University at Buffalo, Buffalo, USA\\
\textsuperscript{3}Department of Computer Vision Technology (VIS), Baidu Inc\\
\textsuperscript{*}Corresponding author, email: bczhang@buaa.edu.cn, hanshumin@baidu.com}

\maketitle
\ificcvfinal\thispagestyle{empty}\fi

\begin{abstract}
   Model ensembles are becoming one of the most effective approaches for improving object detection performance already optimized for a single detector. Conventional  methods directly fuse bounding boxes but typically fail to consider proposal qualities when combining detectors. This leads to a new problem of \textbf{confidence discrepancy} for the detector ensembles.  The confidence has little effect on single detectors but significantly affects detector ensembles.  To address this issue, we propose a novel ensemble called the Probabilistic Ranking Aware Ensemble (PRAE) that refines the confidence of bounding boxes from detectors. By simultaneously considering the category and the location on the same validation set, we obtain a more reliable confidence based on  statistical probability.  We can then rank the detected bounding boxes for assembly.  We also introduce a bandit approach to address the confidence imbalance problem caused by the need to deal with different numbers of boxes at different confidence levels.  We use our PRAE-based non-maximum suppression (P-NMS) to replace the conventional NMS method in ensemble learning. Experiments on the PASCAL VOC and COCO2017 datasets demonstrate that our PRAE method consistently outperforms state-of-the-art methods by significant margins.
\end{abstract}

\section{Introduction}

Object detection is one of the hottest topics in the field of computer vision. We have recently witnessed significant success thanks, in part, to the unprecedented representation capacity of Convolutional Neural Networks (CNNs) ~\cite{iff,RCNN14,FastRCNN15,FasterRCNN15,YOLO16,Redmon_2019,FPN17,SSD16}. In a recent survey~\cite{Survey2019}, the authors categorize numerous approaches to CNN-based object detection as  one-stage~\cite{FasterRCNN15} vs. two-stage~\cite{FocalLoss17} architectures, single-scale ~\cite{FasterRCNN15} vs. pyramid feature networks~\cite{FPN17}, and anchor-based~\cite{SSD16} vs. anchor-free~\cite{NAS-FPN2019,DETR2020,ACT2020,smca2020} techniques. These approaches focus on model refinement for a single detector, but their performance is often stifled by their model's capacity for realistic tasks \cite{tiny-person2020}.

As demonstrated by numerous engineering attempts and visual object detection contests \cite{tiny-person2020}, single detector approaches still do not perform well enough for real applications.  It is our belief, however that detector ensemble approaches can push the boundaries of object detection performance to acceptable levels. Ensembles are also theoretically supported by the fact that when none of the models in an ensemble can detect objects accurately enough, there often exists a combination (ensemble) that can perform better than any individual model. Based on these observations, model ensemble strategies are widely used  to meet practical requirements of industrial applications. 


\begin{figure}[t]
	\centering
	\subfigure[]{\includegraphics[height=2.6cm]{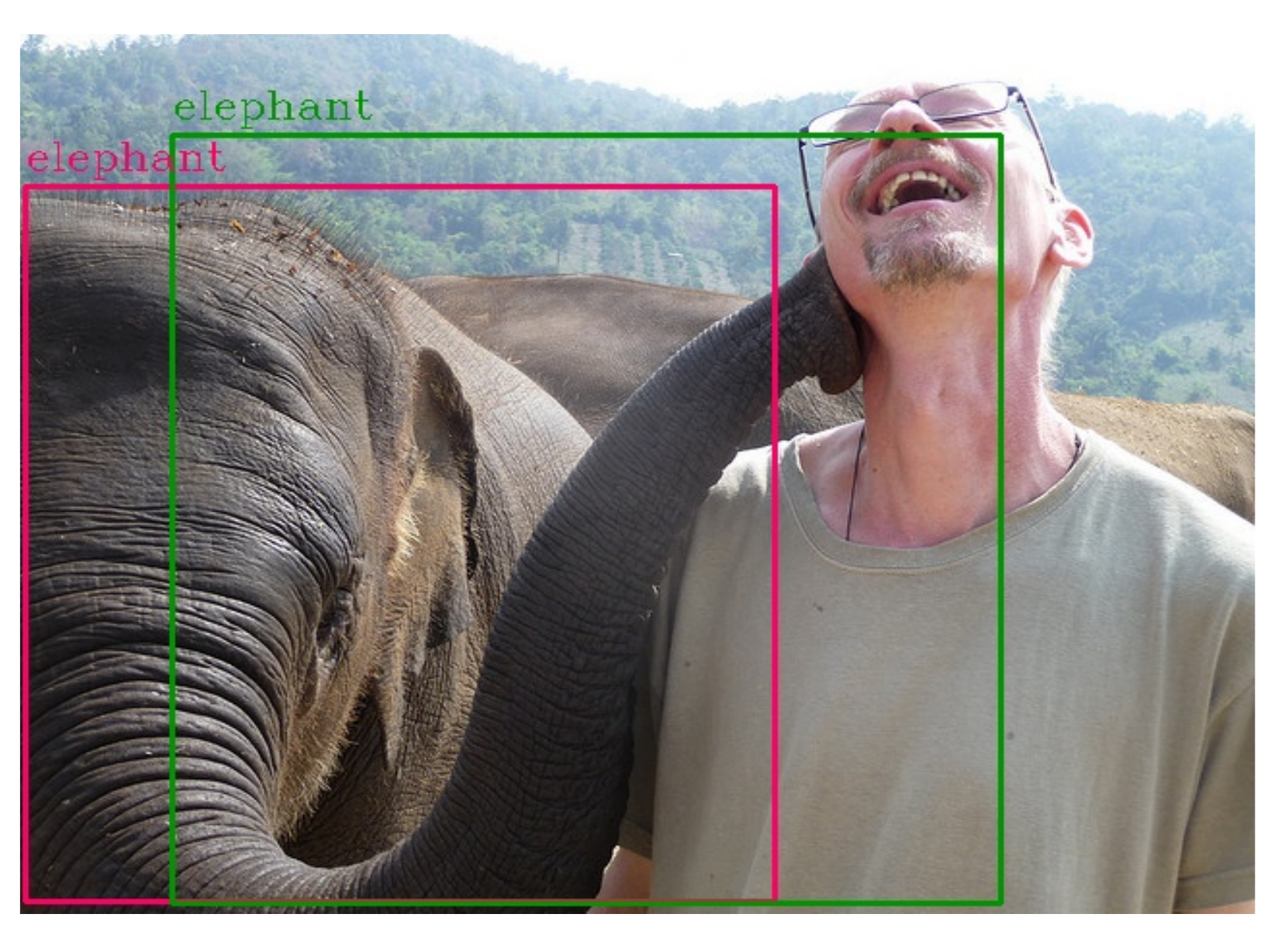}}
	\subfigure[]{\includegraphics[height=2.7cm]{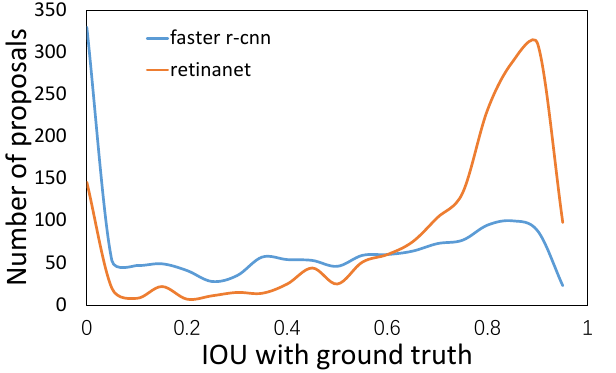}} \\
	\caption{Proposals from different detectors with the same confidence have vastly different localization qualities. a) Faster R-CNN~\cite{FasterRCNN15} (red) and Retinanet~\cite{FocalLoss17} (green) bounding boxes, each with a  confidence of 0.7; b) The number of proposals with confidence 0.7 as a function of the IOU with the ground truth on the COCO2017~\cite{Lin2014MicrosoftCC} validation set. (best viewed in color)}
	\label{intro}
\end{figure}
\begin{figure*}[h]
		\begin{center}
			\centerline{\includegraphics[scale=0.18]{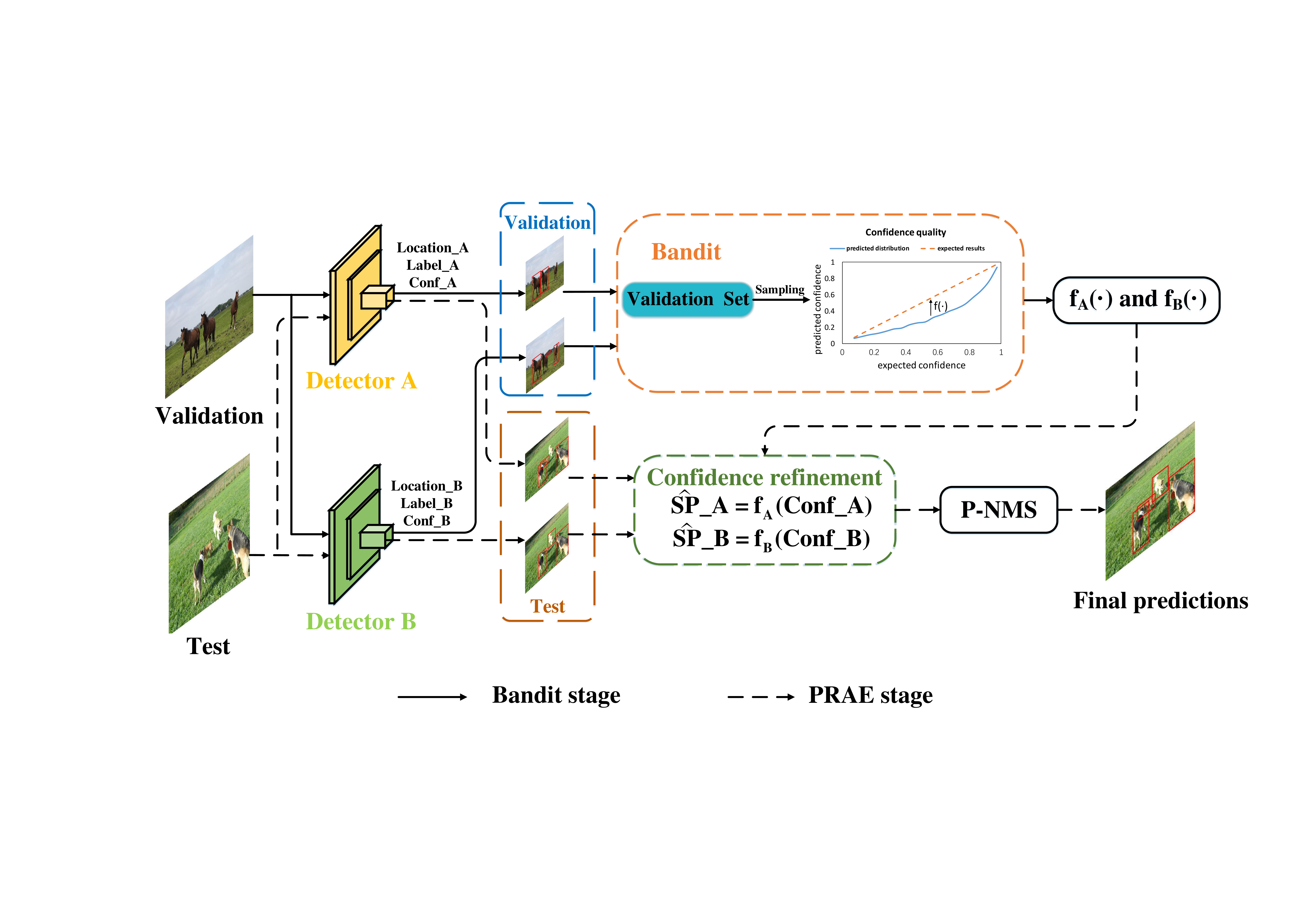}}
			\caption{The architecture of the PRAE ensemble. $f(\cdot)$ represents a mapping relationship between confidence score and statistical probability, $\hat{SP}$ denotes the refined confidence called statistical probability, $Conf$ denotes the original predicted confidence, and $P-NMS$ denotes the PRAE-based NMS. (best viewed in color)}
			\label{architecture}
		\end{center}
\end{figure*}

Results from object detection models typically include locations (bounding boxes), category labels, and associated confidences.
Prevailing ensemble methods in object detection are  based on a simple assembly of  bounding boxes predicted by the different models.  Methods typically use  Non-Maximum Suppression (NMS)~\cite{nms2006} or Soft-NMS~\cite{soft-nms2017} to remove significantly overlap boxes and update box confidences. Following a similar path, Non-maximum weighting (NMW)~\cite{nmw2017,nmw2017-2} uses the IOU value to weigh the boxes but does not change the confidence scores. Weighted
Boxes Fusion (WBF)~\cite{wbf2019} refines the bounding box coordinates with significant overlap.  They use a voting strategy to compute a weighted average of their locations and tune the weights assigned by different models. The method achieves a good performance but  suffers from a tedious parameter tuning process.  The weights for the ensemble are greedily traversed over all potential  combinations, and this number grows exponentially with the number of models.

Unfortunately, these  methods neglect the fact that   the proposals from different detectors with the same confidence may have vastly different localization qualities, as shown in Fig.~\ref{intro}.   This leads to a \emph{\textbf{confidence discrepancy problem}} (details are given in Section 3.1), which significantly affects the detector ensemble. A possible reason for this is that features extracted by the backbone are shared by the location regressor and classifier in the head~\cite{regcls2019}. Regression tasks tend to focus on the center point of an object, while classification tasks focus on salient parts~\cite{tian2019fcos}.    For example, the classification task  may consider the head of a person, which is not always at the center of the detection. However,   the confidences  are  mostly category-related. This explains why the proposal qualities are different even for the same confidence.  In other words, when the confidence is used to rank the bounding boxes in existing ensemble methods, the results are less than ideal.  While the category-related  confidence discrepancy is reduced, we can fairly rank the bounding boxes and eliminate this bottleneck for the ensemble. 


In this paper, a probabilistic ranking-aware ensemble (PRAE) method is introduced  to refine the confidence score by considering both category and localization information on the same dataset. We first investigate  the rationale of the detector ensemble, which shows  that  different detectors have different  proposal qualities.  This leads to the  confidence discrepancy that generally has little effect on a single detector but significantly affects ensemble detectors.  To address this issue, our PRAE refines the predicted confidence by randomly sampling bounding boxes on the same validation dataset, obtaining more reliable confidence called statistical probability. In addition, we  introduce the bandit approach (Upper Confidence Bound) to improve the confidence calculation and address the imbalanced confidence problem. Finally, we assemble the refined results and adjust the statistical probability of significantly overlapped bounding boxes using PRAE-based non-maximum suppression (P-NMS). We evaluate various detectors on the PASCAL VOC~\cite{voc2015} and COCO2017~\cite{Lin2014MicrosoftCC} datasets to demonstrate the general applicability and effectiveness of our PRAE strategy.  Without a tedious parameter tuning process, our algorithm  consistently improves the baselines by a large margin. The contributions of this paper are summarized as follows:
\begin{enumerate}
    \item  We are the first, to the best of our knowledge, to handle the confidence discrepancy problem caused by  proposal qualities of different detectors for ensemble models.
    \item  We present a novel and effective ensemble method called probabilistic ranking-aware ensemble (PRAE), which refines the confidence based on statistical probability by considering both category  and  localization  information on  the  same dataset.
    \item Without tedious parameter tuning, our PRAE method outperforms existing ensemble algorithms by a significant margin.
\end{enumerate}

\section{Related Work}
\textbf{Non-maximum suppression (NMS).} Predictions of a detector consist typically of locations (bounding boxes), category labels, and associated confidences. For human detection, Dalal and Triggs~\cite{humandetection} demonstrated a greedy NMS algorithm where the bounding box with the maximum detection score is selected, and its neighboring boxes
are suppressed using a pre-defined overlap threshold. Since then, NMS has been widely adopted in object detection to reduce the occurrence of false positives. The approach sorts all of the predicted bounding boxes according to their confidence and selects box with the maximum confidence. All other detection boxes with a significant overlap (using a pre-defined threshold) are suppressed. This process is recursively applied to the remaining boxes, and the selected boxes are the final predictions. However, a major issue with NMS is that it sets the score for neighboring detections to zero. Thus, if an object was actually present in that overlap threshold, it would be missed, and this would lead to a reduction in mAP.

\textbf{Soft-NMS.} Soft-NMS~\cite{soft-nms2017} is proposed to address the issue of missing close detections. Instead of directly suppressing the boxes that are significantly overlapped with the selected box by setting their score to zero, Soft-NMS decreases the detection scores as an increasing function of overlap.  Intuitively, if a bounding box has a significant overlap with a box having a higher score, it should take on the higher score, and if the overlap is low, the box can maintain its original score. Soft-NMS give those boxes that would be suppressed during NMS another chance to be selected.  Soft-NMS has resulted in noticeable improvements over traditional NMS on standard benchmark datasets, like PASCAL VOC~\cite{voc2015} and COCO2017~\cite{Lin2014MicrosoftCC}. However, neither NMS nor Soft-NMS refines the locations of bounding boxes. 


\textbf{Non-maximum weighting (NMW) and Weighted Boxes Fusion (WBF).} NMW~\cite{nmw2017,nmw2017-2} presents that lower confident bounding boxes may consider some latent information that is ignored by the most confident boxes. Especially in the case that two boxes have similar confidences, predicting the average box is more convincing than directly choosing the higher one. Thus, NMW weighs boxes according to their IOU with the box with the highest confidence, updates the locations by calculating the weighted average of their locations and then implements the NMS method without updating the box confidence. However, NMW does not use information about how many models predict a given box in a cluster and, therefore, does not produce the best results for model ensemble. By contrast, WBF~\cite{wbf2019} updates a fused box at each step, and uses it to check the overlap with the next predicted boxes. Specifically, it uses confidence scores of all proposed bounding boxes to construct an average box. When several boxes overlap significantly, WBF takes the average of their confidences as the final confidence and takes the weighted average of coordinates as the final location. The strategy improves the quality of the combined predicted bounding boxes and achieves a better mAP than other methods.

However, these methods only focus on the box fusion process and the direct weight assignment for different models but fail to deal with the confidence discrepancy problem before assembly. Besides, for NMW~\cite{nmw2017,nmw2017-2} and WBF~\cite{wbf2019},  weights for the ensemble are greedily traversed over all potential  combinations, and this number grows exponentially with the number of models. This leads to a tedious parameter tuning process.

\section{Probabilistic Ranking-Aware Ensembles (PRAE) }

In this section, we first describe the confidence discrepancy problem in existing detector ensemble methods. We then present our PRAE method, as shown in Fig.~\ref{architecture} and explain how it works. 

\subsection{Confidence Discrepancies in Detector Ensembles}

For object detection, the mAP  is the most commonly used evaluation metric. During the inference stage, the results include locations (bounding boxes), category labels, and associated confidences, and the mAP is acquired by sorting the bounding boxes in decreasing order of confidence, obtaining a precision-recall curve, and computing the area bounded by the curves and the axis. 

However, we find that sorting bounding boxes in terms of their confidence has two main drawbacks for ensembles.  First, the confidence is most often calculated with a normalization strategy using the Softmax function based on a Maximum Entropy Model. Thus, there is typically a deviation from the actual value. Second,  the confidence cannot precisely measure the classification and localization quality at the same time, and thus may not be the best sorting criteria for model ensembles.

To address these two issues, we  refine the  confidence of each bounding box to statistical probability (shortened to $SP$) using our PRAE method.  Statistical probability measures the possibility that this bounding box correctly matches a given target object. A correct match requires two conditions. First,  the Intersection over Union (IOU) of the bounding box and the target object must be  greater than a pre-defined threshold (usually 0.5).  Second, the category of the bounding box must be correct. We will show that sorting the $SP$ value instead of the confidence can achieve higher mAP after ensemble. The detailed calculation of $SP$ is given in \textbf{Section 3.2}.

The $SP$ value can be roughly estimated by sampling  the validation set. We choose several detectors with different mAPs on the COCO2017 validation set~\cite{Lin2014MicrosoftCC}. At each confidence level, we compute the ratio of the correct matches to  predicted bounding boxes. The results in Fig.~\ref{score}(a) shows the significant discrepancy between the confidence and the $SP$ values. More importantly, the discrepancies differ across  detectors. One possible  reason is that different detectors have different feature extraction capacities.  The proposal's inputs to the classifiers vary in quality, which leads to discrepancies between the confidence and the $SP$ values shown in Fig.~\ref{score}(a). In fact, such inconsistencies are one of the most dominant factors that deteriorate ensemble performance.

\begin{figure}[h]
	\centering
	\subfigure[]{\includegraphics[height=2.7cm]{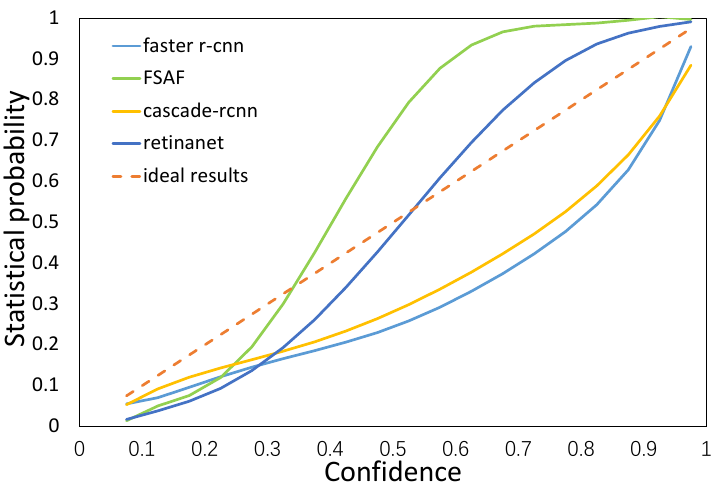}}
	\subfigure[]{\includegraphics[height=2.7cm]{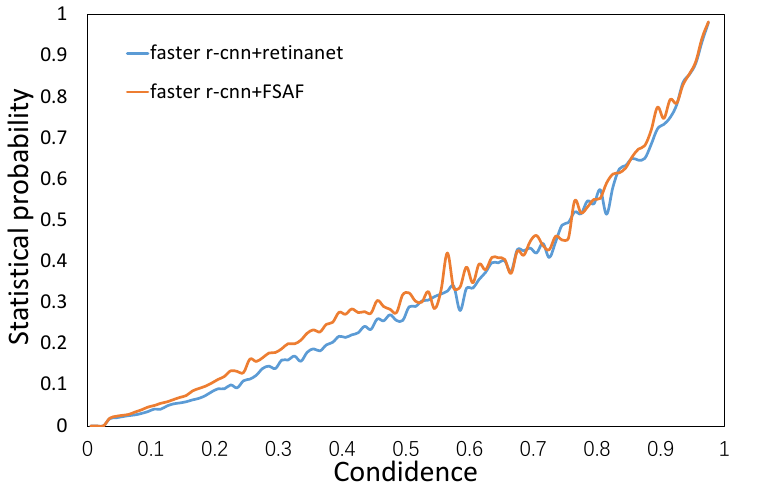}} \\
	\caption{(a) Discrepancies between the confidence and statistical probability of single detectors; (b) discrepancies between the confidence and statistical probability after ensemble. The ensemble method we use is WBF~\cite{wbf2019}.}
	\label{score}
\end{figure}

\begin{theorem}
    For any predictions (including locations, category labels, and associated confidences), if the ranks of confidence are consistent with ranks of the probability, the expectation of the mAP can be maximized.
\end{theorem}

\begin{proof}
	\label{proof}
	The mAP is defined as the area bounded by the Precision-Recall (PR) curves and axis, which can be approximately calculated as 
    \begin{equation}
    	\begin{aligned}
    	mAP = \frac{1}{N}\sum_{n=1}^{N}Pr(R_{n}),
    	\end{aligned}
    	\label{PR-1}
    \end{equation}
    and
    \begin{equation}
    	\begin{aligned}
    	R_{n} = \frac{1}{N} \times n,
    	\end{aligned}
    	\label{PR-2}
    \end{equation}
    where $Pr(R_{n})$ represents the corresponding precision value when the  recall value equals $R_{n}$, and generally $N = 100$. Note that the final mAP is the mean  of the result with different IOU thresholds. To facilitate the analysis, Equ.~\ref{PR-1} is given based on a specific threshold. 
    
    Suppose a certain prediction including locations (bounding boxes), categories, and confidences  is denoted as
    \begin{equation}\nonumber
	\begin{aligned}
	S = \{s_{i}, i=1,2,\cdots, N\},
	\end{aligned}
	\end{equation}
	and
	\begin{equation}\nonumber
	\begin{aligned}
	s_{i} = (bbox_{i}, conf_{i}, cls_{i}).
	\end{aligned}
	\end{equation}
	 $S$ is sorted based on the confidence from the highest to  the lowest. 
	
	\textbf{A proof by contradiction.} We prove our theorem by contradiction and assume  that $S$ can achieve the highest  expectation of mAP, even when the ranks of confidence and probabilities are not consistent. The assumption means that there exists $m$ and $n$, satisfying $conf_{m} > conf_{n}$ and $prob_{m} < prob_{n}$. Then, we have four possible cases, as shown in Table~\ref{situations}.
	\begin{table}[]
    	\begin{center}
    		\begin{tabular}{cccc}
    			\Xhline{1pt}\noalign{\smallskip}
    			situation &$s_{m}$ &$s_{n}$ &prob\\
                \noalign{\smallskip}
                \Xhline{1pt}
                \noalign{\smallskip}
                1 &True &True &$P(s_{m})P(s_{n})$\\
                \hline
                2 &True &False &$P(s_{m})(1-P(s_{n}))$\\
                \hline
                3 &False &True &$(1-P(s_{m}))P(s_{n})$\\
                \hline
                4 &False &False &$(1-P(s_{m}))(1-P(s_{n}))$\\
                \Xhline{1pt}
    		\end{tabular}
    	\end{center}
    	\caption{Four cases  of two predicted boxes. $True$ denotes that the predicted box matches a target object, while $False$ denotes  a false detection.}
    	\label{situations}
    \end{table}
    Then we calculate the  expectation of the mAP as 
    \begin{equation}
    	\begin{aligned}
    	E(mAP) = \sum_{i=1}^{N=4}prob_{i} \times mAP_{i},
    	\end{aligned}
    	\label{expectation}
    \end{equation}
    where $prob_{i}$ and $mAP_{i}$ denote the probability and mAP under the $i^{th}$ situation. For the ranked box list $S$, the predicted boxes are sorted from the highest to the lowest in terms of their confidence, so $s_{m}$ ranks higher than $s_{n}$. Thus, we have
    \begin{equation}
    	\begin{aligned}
    	S_{1}=\{s_{1},s_{2},...,s_{m-1},s_{m}^{T},s_{m+1},...,s_{n-1},s_{n}^{F},...\},
    	\end{aligned}
    	\label{S1}
    \end{equation}
    where $s_{m}^{T}$ denotes the $m^{th}$ box is True, while $s_{n}^{F}$ denotes the $n^{th}$ box is False, which represents  Case 2 in Table~\ref{situations}. Similarly, we have
    \begin{equation}
    	\begin{aligned}
    	S_{2}=\{s_{1},s_{2},...,s_{m-1},s_{m}^{F},s_{m+1},...,s_{n-1},s_{n}^{T},...\},
    	\end{aligned}
    	\label{S2}
    \end{equation}
    which represents  Case 3 in Table~\ref{situations}. we define $S^{i}$ as
    \begin{equation}
    	\begin{aligned}
    	S^{i}=\{s_{1},s_{2},...,s_{i}\},
    	\end{aligned}
    	\label{Si}
    \end{equation}
    Obviously, When $i\le m-1$ or $i \ge n$, $S_{1}^{i}$ and $S_{2}^{i}$ have the same number of $True$ and $False$ boxes. However, when $m\le i\leq n$, as long as $S_{1}^{i}$ and $S_{2}^{i}$ have the same recall rate $R_{i}$, we have $P_{r1}(R_{i}) > P_{r2}(R_{i})$, because $S_{1}^{i}$ always have one less $False$ box than $S_{1}^{i}$. Thus, according to Equ.~\ref{PR-1}, $S_{1}$ achieves higher mAP than $S_{1}$ does, and we have $mAP_{2} > mAP_{3}$.
    
    Based on the assumption that $P(s_{m}) < P(s_{n})$ we have $prob_{2} < prob_{3}$ in Equ.~\ref{expectation}. if we appropriately adjust the confidence of $s_{m}$ and $s_{n}$, and exchange their positions  in $S$, and have 
    \begin{equation}
    	\begin{aligned}
    	\hat{S}=\{s_{1},s_{2},...,s_{m-1},s_{n},s_{m+1},...,s_{n-1},s_{m},...\},
    	\end{aligned}
    	\label{S'}
    \end{equation}
    and then the second and  third items in Equ.~\ref{expectation} are changed from
    \begin{equation}
    	\begin{aligned}
    	S:prob_{2} \times mAP_{2} + prob_{3} \times mAP_{3}
    	\end{aligned}
    	\label{tmp1}
    \end{equation}
    to 
    \begin{equation}
    	\begin{aligned}
    	\hat{S}:prob_{2} \times mAP_{3} + prob_{3} \times mAP_{2}.
    	\end{aligned}
    	\label{tmp1}
    \end{equation}
 {  Obviously, $\hat{S}$ achieves a higher mAP expectation than $S$.  The \textbf{theorem} is proved. } 
\end{proof}

\begin{corollary}
The bounding boxes' confidences become inconsistent with their probability  for  ensemble, leading  to a sub-optimal  detector. The reasons are highlighted in our \textbf{theorem}.  Only when the sorted bounding boxes maintain a probability  consistent with the confidence,  a higher mAP can be achieved.
\label{corollary}
\end{corollary}

 As shown in Fig.~\ref{score}(a), there exist obvious confidence discrepancies between detectors. Thus, the confidence of bounding boxes predicted by different detectors has to be refined.  We test the conventional ensemble strategy to combine detectors, and the results ( in Fig.~\ref{score}(b))   show that the  confidence and statistical  probabilities are significantly inconsistent, which severely deteriorates the ensemble performance. Thus, according to the \textbf{theorem}, we introduce the corollary above as  a basic principle  to improve the performance of the ensemble detector.
 


\subsection{Probabilistic Ranking-aware  Confidence Refinement}
To bridge the gap between the confidence and the probability, we introduce  statistical probability  as a refined measure of the confidence of each box. We refine  the confidence of the bounding boxes based on a validation set shared by all detectors.  The result is a fair confidence measure based on statistical probability. The probability that a bounding box correctly matches a target object is a conditional probability, denoted as $P(bbox|conf)$.
To estimate $P(bbox|conf)$, we sample on the validation set. We first quantize the original confidence to discrete levels by dividing the confidence interval (0$\sim$1) into  sub-intervals with a length  $d$ and a centre point  defined as 
\begin{equation}
	\begin{aligned}
	c_{i} = d \times i - \frac{d}{2}, 1\le i \le ceil(\frac{1}{d}),
	\end{aligned}
	\label{quantization-1}
\end{equation}
and the sub-intervals are denoted as
\begin{equation}
	\begin{aligned}
	l_{i} = [c_{i} - \frac{d}{2},c_{i} + \frac{d}{2}).
	\end{aligned}
	\label{quantization-2}
\end{equation}
In each sub-interval $l_{i}$, we count the number of all bounding boxes $T_{i}$ and correct matches  $TP_{i}$.  We then compute the ratio of $TP_{i}$ to $T_{i}$ and obtain statistical probability as
\begin{equation}
	\begin{aligned}
	SP_{i} = \frac{TP_{i}}{T_{i}},
	\end{aligned}
	\label{quantization-3}
\end{equation}
which is a rough estimate of $P(bbox|conf)$ in sub-interval $l_{i}$. However, $T_{i}$ differs greatly across confidence levels, as shown in Fig.~\ref{curve}, leading to a confidence imbalance problem. Inspired by bandit~\cite{bandit1995}, the number of samples represents the exploitation level. The more samples in a sub-interval $l_{i}$, the more reliable  $SP_{i}$. Thus, we introduce Upper Confidence Bound (UCB) to refine $SP_{i}$, which guarantees that every sub-interval is fairly considered as follows: 
\begin{equation}
	\begin{aligned}
	SP_{i}^{\ast}  = SP_{i} + \theta \sqrt{\frac{2\ln{\sum_{i=1}^{N}T_{i}}}{T_{i}}}.
	\end{aligned}
	\label{ucb}
\end{equation}
where $SP_{i}^{\ast}$ denotes the refined statistical probability in $l_{i}$, $\theta$ is a tunable parameter, $N$ denotes the number of the sub-intervals. 

\begin{figure}[h]
	\centerline{\includegraphics[height=4.4cm]{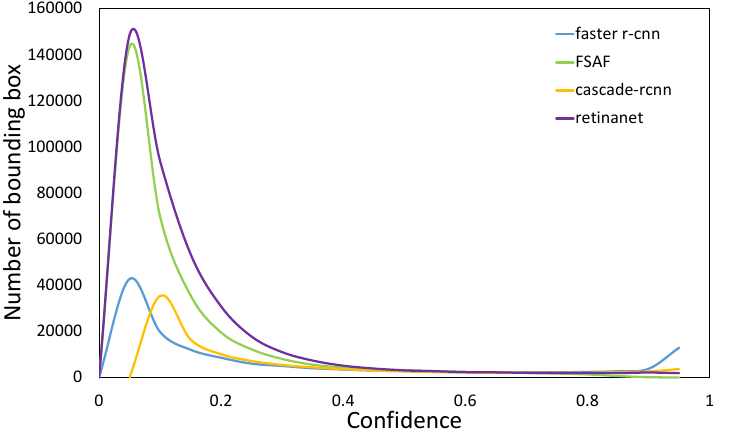}}
	\caption{Quantity distribution of bounding boxes across different confidence levels.}
	\label{curve}
\end{figure}

However, $SP^{\ast}$ is discrete, which results in quantization error. Obviously, in a certain sub-interval $l_{i}$, it is unfair to transform all confidence values to a single value $SP_{i}^{\ast}$, because according to Fig.~\ref{score}(a), the ranks of the confidence and probability are consistent.  This means that the bounding boxes with higher confidences are more likely to locate a target object. Thus, we propose a ranking-aware strategy to further refine the previous  $SP_{i}^{\ast}$ by simultaneously considering the confidence and ranks of bounding boxes. The final statistical probability of the bounding box (denoted as $\hat{SP}$) with the confidence is obtained by
\begin{equation}
	\begin{aligned}
	\hat{SP}_{i} &= f(confidence)\\
	                  &= SP^{\ast}_{i} \times \frac{confidence}{rk(confidence)},
	\label{quantization-5}
	\end{aligned}
\end{equation}
and
\begin{equation}
	\begin{aligned}
	rk(confidence) = 
	\left\{
	\begin{array}{lr}
	1, & conf \in l_{1},\\
	2, & conf \in l_{2},\\
	\vdots\\
	N, & conf \in l_{N}.
	\end{array}
	\right.
	\end{aligned}
	\label{quantization-2.5}
\end{equation}
where $rk(confidence)$ denotes the rank of the confidence. After probabilistic ranking-aware  confidence refinement,  the ranks of $\hat{SP}$ and statistical probability are consistent in the  ensemble. This  helps to improve the ensemble performance, as suggested by the \textbf{theorem in Section 3.1}. Statistical probability achieves confidence refinement globally and guarantees the consistency of confidence and ranks locally, resulting in a more rational indicator for sorting the bounding boxes during evaluation.

\subsection{PRAE-based NMS (P-NMS)}
The overlapping bounding boxes from different detectors  are generally handled by an NMS process that directly selects the box with the highest confidence and removes others. However, considering the statistical probability values have reduced the confidence discrepancies between  detectors,  all overlapping boxes can provide equally important information and should be involved in the calculation.  To handle the issue, we introduce a PRAE-based NMS   to  effectively fuse these overlapping  bounding boxes. To do this, when the IOU of several boxes  is greater than a threshold, we take the average of their $\hat{SP}$ values.  Then, the location is updated by 
\begin{equation}
	\begin{aligned}
	loc = \sum_{k=1}^{n}\hat{SP}_{k} \times loc_{k},
	\label{P-NMS}
	\end{aligned}
\end{equation}
where $n$ denotes the number of overlapping boxes, $\hat{SP}$ and $loc_{k}$ denote the refined statistical probability and the location of the $k^{th}$ bounding box. Note that $loc$ is a 4-dimensional vector, which contains the co-ordinates of top-left and bottom-right points of the bounding box.

The validation and test set are assumed to be independent identically distributed.  As a result, $\hat{SP}$  obtained  on the validation set are also effective on the test set, as verified in \textbf{Section 4}. Our PRAE approach is summarized in Algorithm \ref{alg}.
\begin{algorithm}[tb]
	\caption{The PRAE method}
	\label{alg}
	\begin{algorithmic}
	    \STATE $\#$ Note: conf is short for confidence
		\STATE {\bfseries Input:}
		$val$: validation set.\\
		\qquad \quad$S_{A}^{v}$: Boxes predicted on $val$ by detector A.\\
		\qquad \quad $S_{B}^{v}$: Boxes predicted on $val$ by detector B.\\
		\qquad \quad $S_{A}^{t}$: Boxes predicted on the test set by detector A.\\
		\qquad \quad $S_{B}^{t}$: Boxes predicted by the test set by detector B.\\
		\STATE {\bfseries Output:} $S_{f}$ : Final predictions on the test set.
		\STATE $S_{f} \gets \phi$ 
		\FOR{$S$ in $(S_{A}^{v},S_{B}^{v})$}
		\STATE $c_{i},l_{i}, rk(conf) \gets $ Quantize the confidence interval.
		\STATE $SP_{i} \gets $ Sample on the $val$.
		\STATE $SP_{i}^{\ast} \gets$ Implement bandit.
		\STATE $S^{t} \gets S_{A}^{t}$ If $S$ is $S_{A}^{v}$ Else $S_{B}^{t}$
		    \FOR{$s$ in $S^{t}$}
		    \STATE $rk(s) \gets$ based on Equ.~\ref{quantization-2.5}.
		    \STATE $\hat{SP}(s) \gets$ based on Equ.~\ref{quantization-5}
		    \STATE $S_{f} \gets S_{f} \cup \hat{SP}(s)$
		    \ENDFOR
		\ENDFOR
		\STATE $S_{f} \gets P-NMS(S_{f})$
		\STATE Return $S_{f}$
	\end{algorithmic}
\end{algorithm}

\section{Experiments}
In this section, we implement our PRAE method to evaluate the effectiveness of the proposed statistical probability. We also compare our method with other state-of-the-art approaches. The performance of the PRAE method is evaluated on various detectors using the PASCAL VOC~\cite{voc2015} and COCO2017~\cite{Lin2014MicrosoftCC} datasets. 

\subsection{Implementation details}
To fairly compare with other ensemble methods, we evaluate our PRAE with detectors achieving different mAP levels. For COCO dataset, all selected models are trained on the COCO2017 training set (118k images)and optimized on the validation set (5k images), while for PASCAL VOC dataset, models are trained on VOC07+12 training set (14k images) and optimized on the VOC2007 validation set (2k images).


\begin{table}[h]
    \begin{center}
        \begin{tabular}{p{1cm}<{\centering} p{1.1cm}<{\centering}p{1.1cm}<{\centering}p{1.1cm}<{\centering}p{1.1cm}<{\centering}p{1.1cm}}
			\Xhline{1pt}
			$d$ &0.01 &0.03 &0.05 &0.07\\
			mAP  &44.4 &45.9 &46.1 &46.0\\
			\hline
			$\theta$ & 0   & 0.5    & 1  &1.5 \\
			mAP  &46.1 &46.5 &46.7 &46.5\\
			\Xhline{1pt}
		\end{tabular}
    \end{center}
    \caption{Comparison of different $d$ in Equ.~\ref{quantization-1} and $\theta$ in Equ.~\ref{ucb}. $d$ is set to 0.05 when $\theta$ varies. mAP denote mAP@(0.5..0.95) on the COCO validation set.}
    \label{ablation-1}
\end{table}

\begin{table*}[]
    \begin{center}
        \begin{tabular}{lccc}
            \Xhline{1pt}
            Model  &\makecell[c]{Retinanet~\cite{FocalLoss17} (mAP 37.5) \\ \& \\Faster R-CNN~\cite{FasterRCNN15} (mAP 39.2)} &\makecell[c]{FCOS~\cite{tian2019fcos} (mAP 42.5)\\ \& \\Mask R-CNN~\cite{MaskRCNN17} (mAP 43.6)}
            &\makecell[c]{Cascade Mask~\cite{cascade18} (mAP 50.3)\\ \& \\CascadeClsAware (mAP 51.7)}\\
            \Xhline{1pt}
            NMS~\cite{nms2006} &39.1 & 44.6 & 52.6 \\
            IOU threshold & 0.7 & 0.65 &0.7\\
            weight & [3, 4] & [2, 3] & [1, 2]\\
            \hline
            Soft-NMS~\cite{soft-nms2017} & 39.1 & 44.7 & 52.3 \\
            IOU threshold & 0.7 & 0.7 &0.7\\
            sigma & 0.1 & 0.1 & 0.1\\
            weight & [1, 4] & [4, 5] & [3, 5]\\
            \hline
            NMW~\cite{nmw2017,nmw2017-2} & 39.8 & 45.7 & 52.8 \\
            IOU threshold & 0.75 & 0.7 &0.7\\
            weight & [2, 3] & [3, 4] & [1, 2]\\
            \hline
            WBF~\cite{wbf2019} & 39.0 & 45.2 & 52.8 \\
            IOU threshold & 0.7 & 0.7 &0.7\\
            weight & [1, 4] & [2, 3] & [1, 2]\\
            \hline
            \textbf{PRAE (ours)} &\textbf{40.8} &\textbf{46.7} &\textbf{53.4}\\
            IOU threshold & 0.7 & 0.7 &0.7\\
            weight & [1, 1] & [1, 1] & [1, 1] \\
            \Xhline{1pt}
        \end{tabular}
    \end{center}
    \caption{Ensemble results of two models on the COCO2017 validation set.} 
    \label{2-model}
\end{table*}

\subsection{Ablation study}
We conduct ablation experiments to evaluate the effect of the parameters.
The tunable parameters of the PRAE method are $d$ in Equ.~\ref{quantization-1} and $\theta$ in Equ.~\ref{ucb}, which drives the effect of the bandit. We choose FCOS (ResNext101-FPN)~\cite{tian2019fcos} with 42.5\% mAP and Mask R-CNN (Res2Net101-FPN~\cite{zhu2019feature} with 43.6\% mAP to ensemble. We vary $d$ and $\theta$ separately, and summarize the results evaluated on the COCO2017 validation set in Table~\ref{ablation-1}. To confirm the effect of different parts of the PRAE strategy, we conduct another ablation study, as shown in Table~\ref{ablation-2}.
\begin{table}[h]
    \begin{center}
        \begin{tabular}{ccccc}
            NMS & PRAE &Bandit &P-NMS &\makecell[c]{mAP@(0.5..0.95) \\ (validation)}\\
            \Xhline{1pt}
            $\checkmark$  & &   & &44.9  \\
			$\checkmark$ & $\checkmark$ & & &45.5\\
			$\checkmark$ & $\checkmark$ &$\checkmark$ & &45.9\\
			& $\checkmark$ & $\checkmark$ &$\checkmark$ &46.7
        \end{tabular}
    \end{center}
    \caption{The effects of different parts of the PRAE. The IOU thresholds for NMS and P-NMS are set to 0.7, $d$ is set to 0.05 and $\theta$ for bandit is set to 1.}
    \label{ablation-2}
\end{table}

In Table~\ref{ablation-1}, we find that bandit with an appropriate $\theta$ value plays a positive role in the ensemble.  This demonstrates its ability to address the confidence imbalance problem. $\theta$ is set to 1 in the following experiments. In Table~\ref{ablation-2}, the sorting of the predicted bounding boxes by $\hat{SP}$ instead of confidence helps to achieve a 0.8\% improvement in terms of mAP.  This strongly suggests that the proposed statistical probability is a much better indicator of the probability that a bounding box correctly matches a target object  than the confidence. Unlike existing methods, P-NMS does not assign tunable weights to different detectors.  But it still helps to achieve considerable improvement.  This improvement  further verifies that PRAE eliminates the confidence discrepancy of detectors, making them equally reliable and thus avoiding the parameter tuning process. 

\subsection{An Ensemble of two  models}
Existing ensemble methods need to weigh different models before assembling them. The ensemble performance is very sensitive to the pre-defined weights. To eliminate the effects of these weights  as much as possible, we implement an ensemble of only two models to fairly compare our PRAE method. We test our PRAE method on the PASCAL VOC2007~\cite{voc2015} test set and the COCO2017~\cite{Lin2014MicrosoftCC} validation set. 

For the VOC dataset, we choose Faster R-CNN~\cite{FasterRCNN15} (ResNet50-FPN as a backbone and 1x training schedule) and Retinanet ~\cite{FocalLoss17} (ResNet50-FPN as a backbone and 1x training schedule) for ensemble. The results are shown in Table~\ref{2-model-voc}.
\begin{table}[]
    \begin{center}
        \begin{tabular}{lccc}
            \Xhline{1pt}
            Model  & Backbone & $AP_{50} (test)$\\
            \Xhline{1pt}
            Retinanet ~\cite{FocalLoss17} & ResNet50-FPN &69.8 \\
            Faster R-CNN~\cite{FasterRCNN15} & ResNet50-FPN &71.3 \\
            \hline
            Method &Weight &$AP_{50} (test)$\\
            \hline
            NMS~\cite{nms2006}  & [1, 2] & 72\\
            Soft-NMS~\cite{soft-nms2017} & [2, 3] &72\\
            NMW~\cite{nmw2017, nmw2017-2} & [1, 2] &72.1\\
            WBF~\cite{wbf2019} &[1, 4] & 71.1\\
            \textbf{PRAE (ours)} &[1, 1] & \textbf{73.3} \\
            \Xhline{1pt}
        \end{tabular}
    \end{center}
    \caption{Ensemble results of two models on the Pascal VOC2007 test set.}
    \label{2-model-voc}
\end{table}

\begin{table*}[h]
    \begin{center}
        \begin{tabular}{lcccccccc}
            \Xhline{1pt}
            Model & Backbone & Lr schd & $AP$    & $AP_{50}$    & $AP_{75}$    & $AP_S$    & $AP_M$    & $AP_L$\\
            \Xhline{1pt}
            FCOS~\cite{tian2019fcos} & ResNext-101-FPN & 2x & 43.1 & 62.8 & 46.3 & 26.1 & 45.7 & 53.4 \\

            Mask R-CNN & Res2Net101-FPN & 2x & 45.3 & 65.1 & 50.0 & 26.8 & 48.7 & 57.0 \\

            HTC~\cite{htc2019} & Res2Net101-FPN & 20e & 47.9 & 67.2 & 52.3 & 28.2 & 51.0 & 60.6 \\

            Cascade Mask & SENet154-vd-FPN & 1.44x & 50.7 & 69.6 & 55.2 & 31.1 &53.4 & 65.0\\

            CascadeClsAware Faster & ResNet200-vd-FPN & 2.5x & 52.1  & 70.0  & 57.4 & 34.2 & 54.4 & 64.0 \\

            Cascade Faster & CBResNet200-vd-FPN & 2.5x & 53.8 & 71.8 &59.4 & 35.7 & 56.4 & 66.4\\
            \Xhline{1pt}
            Method &Weight  &IOU threshold & $AP$    & $AP_{50}$    & $AP_{75}$    & $AP_S$    & $AP_M$    & $AP_L$\\
            \Xhline{1pt}
            \multirow{2}*{NMS~\cite{nms2006}} & [1, 3, 4, 7, 8, 9], & 0.65 & 54.3 & 72.7 & 60.1 & 36 & 57.1 & 67.2 \\
            & \textcolor{red}{[1, 1, 1, 1, 1, 1]} & \textcolor{red}{0.65} & \textcolor{red}{51.5 } & \textcolor{red}{70.7} & \textcolor{red}{57.4} & \textcolor{red}{33.8} & \textcolor{red}{54.8} & \textcolor{red}{64.1}\\
            \hline
            
            \multirow{2}*{Soft-NMS~\cite{soft-nms2017}} & [1, 3, 4, 7, 8, 9] & 0.65 & 53.9 & 72.6 & 58.9 & 35.5 & 56.4 & 66.6 \\ 
            & \textcolor{red}{[1, 1, 1, 1, 1, 1]} & \textcolor{red}{0.65} & \textcolor{red}{50.8 } & \textcolor{red}{70.7} & \textcolor{red}{56} & \textcolor{red}{33.2} & \textcolor{red}{53.9} & \textcolor{red}{63.4}\\
            \hline
            
            \multirow{2}*{NMW~\cite{nmw2017,nmw2017-2}} & [2, 3, 5, 7, 8, 9] & 0.7 & 54.6 & 73.0 & 61.3 & 36.9 & 57.9 & 67.1 \\
            & \textcolor{red}{[1, 1, 1, 1, 1, 1]} & \textcolor{red}{0.65} & \textcolor{red}{53.4 } & \textcolor{red}{71.3} & \textcolor{red}{59.6} & \textcolor{red}{35.4} & \textcolor{red}{56.9} & \textcolor{red}{66.1}\\
            \hline
            
            \hline
            \multirow{2}*{WBF~\cite{wbf2019}} & [2, 4, 5, 7, 8, 9] & 0.7 & 54.7 & 73.1 & 61.5 & 37.4 & 57.8 & 67.0\\
            & \textcolor{red}{[1, 1, 1, 1, 1, 1]} & \textcolor{red}{0.7} & \textcolor{red}{52.7} & \textcolor{red}{72.4} & \textcolor{red}{58.9} & \textcolor{red}{36.0} & \textcolor{red}{55.9} & \textcolor{red}{64.1}\\
            \hline
            \textbf{PRAE (ours)} & [1, 1, 1, 1, 1, 1] & 0.7 & \textbf{56.0} & \textbf{73.8} & \textbf{61.9} & \textbf{37.6} & \textbf{58.6} & \textbf{68.3}\\
            \Xhline{1pt}
        \end{tabular}
    \end{center}
    \caption{Ensemble results of multiple models of COCO2017 test-dev set. Red parts show the performance of existing methods without weight tuning.} 
    \label{6-model}
\end{table*}

For the COCO dataset, we select three groups of detectors with low, medium, and  high  mAP scores. In the low-mAP group, we choose Retinanet~\cite{FocalLoss17} with 37.5 mAP (ResNet50-FPN as a backbone and 2x training schedule), and Faster R-CNN~\cite{FasterRCNN15} with 39.2 mAP (ResNet50-FPN as a backbone and 2x training schedule).  In the medium-mAP group, we choose FCOS~\cite{tian2019fcos} with 42.5 mAP (ResNext-101-FPN as a backbone and 2x training schedule) and Mask R-CNN~\cite{MaskRCNN17} (Res2Net101-FPN as a backbone and 2x training schedule) with 43.6 mAP.  In the high-mAP group, we choose Cascade Mask (SENet154-vd-FPN as a backbone and 1.44x training schedule) with 50.3 mAP, and CascadeClsAware (ResNet200-vd-FPN-Nonlocal as a backbone and 2.5x training schedule) with 51.7 mAP. The results are shown in Table~\ref{2-model}. 

The results indicate that without tedious parameter tuning, the proposed PRAE method consistently outperforms existing methods by significant margins. Furthermore, the PRAE is effective across different mAP levels and different datasets, which proves its general applicability and effectiveness.

\subsection{An ensemble of multiple models}
We also combine the results of multiple models to verify the efficiency of our method. Specifically, we choose FCOS~\cite{tian2019fcos}, Mask R-CNN~\cite{MaskRCNN17}, and several improved versions of Cascade R-CNN~\cite{cascade18}. These models include one-stage and two-stage detectors with different mAP level, which makes the comparison more convincing. The parameters are optimized on the COCO2017~\cite{Lin2014MicrosoftCC} validation set, and methods are evaluated on the test-dev set. The results in Table~\ref{6-model} show that the performance of existing methods is very sensitive to the weight assignment. When assigning equal weight to every model, the ensemble performance of existing methods is severely deteriorated, as the confidence discrepancies of different models are inconsistent. Thus parameter tuning is indispensable for them. However, the combination of weights tends to grow exponentially with the number of models, which becomes unacceptable. By contrast, by bridging the gap between the confidence and the probability, without tedious parameter tuning, PRAE methods still outperform all baselines and advance the state-of-the-art for ensemble methods significantly. The per-class performance is provided in the supplementary materials.

We also compare our PRAE with other methods when there is a huge domain gap between the validation set and test set, which shows that the PRAE still outperforms other methods (provided in supplementary materials).

\section{Conclusion}
We have proposed an elegant and effective approach, referred to as the Probabilistic Ranking Aware Ensembles (PRAE) method, for visual object detection. The proposed statistical probability refine the predicted confidence by randomly sampling bounding boxes on the same validation dataset, obtaining more reliable confidence , and thus achieve a more rational indicator to measure the qualities of bounding boxes. With PRAE implemented, we have significantly improved the performance of the model ensemble, in  contrast to the baselines and the state-of-the-art methods. The underlying effect is that PRAE bridges the gap between the confidence and the probability, and thus addressing the confidence discrepancy issue of detectors, which leads to improved ensemble results. Our PRAE method provides fresh insight into model ensembles for object detection.

\newpage

{\small
\bibliographystyle{ieee_fullname}
\bibliography{egpaper_final}
}

\end{document}